%% file: matrixarxiv.tex
\documentclass[twoside,11pt]{article}
\input{def}

\usepackage{cancel}
\usepackage{commath}
\usepackage{jmlrwcp2e}
\usepackage{url}
\usepackage{amsmath,dsfont}

\jmlrheading{}{}{}{}{}{}{}

\ShortHeadings{PCA with Gaussian perturbations}{Kot{\l}owski and Warmuth}
\firstpageno{1}

\begin{document}

\title{PCA with Gaussian perturbations}

\author{\name Wojciech Kot{\l}owski \email wkotlowski@cs.put.poznan.pl \\
        \addr Pozna{\'n} University of Technology, Poland
        \AND
        \name Manfred K. Warmuth \email manfred@cse.ucsc.edu \\
        \addr University of California, Santa Cruz
}

\newcommand{\fix}{\marginpar{FIX}}
\newcommand{\new}{\marginpar{NEW}}

\newcommand{\NN}{\mathcal{N}}
\newcommand{\Wtilde}{\widetilde{\boldsymbol{W}}}
\newcommand{\EE}[1]{\mathbb{E}\left[#1\right]}

\editor{}

\maketitle

\begin{abstract}%   <- trailing '%' for backward compatibility of .sty file
Most of machine learning deals with vector parameters.
Ideally we would like to take higher order information
into account and make use of matrix or even tensor parameters. 
However the resulting algorithms are usually inefficient. 
Here we address on-line learning with matrix parameters. 
It is often easy to obtain online algorithm
with good generalization performance if you eigendecompose
the current parameter matrix in each trial (at a cost of $O(n^3)$
per trial). Ideally we want to avoid the decompositions and
spend $O(n^2)$ per trial, i.e. linear time in the size of the matrix data. 
There is a core trade-off between the running time 
and the generalization performance, here measured by the 
regret of the on-line algorithm (total gain of the best off-line predictor 
minus the total gain of the on-line algorithm). 

We focus on the key matrix problem of rank $k$
Principal Component Analysis in $\R^n$ where $k \ll n$. 
There are $O(n^3)$ algorithms that achieve the optimum regret 
but require eigendecompositions. We develop a simple algorithm
that needs $O(kn^2)$ per trial whose regret is off by a small factor of $O(n^{1/4})$.
The algorithm is based on the Follow the Perturbed Leader paradigm.
It replaces full eigendecompositions at each trial by the
problem finding $k$ principal components of
the current covariance matrix that is perturbed by Gaussian noise.
\end{abstract}

\section{Introduction}

In Principal Component Analysis
(PCA), the data points $\x_t\in\RR^n$ are projected 
onto a $k$-dimensional subspace (represented by a rank $k$ projection matrix $\P$).
The goal is to maximize the total squared norm of the projected data points,
$\sum_t \|\P\x_t\|^2$. This is equivalent
to finding the {\em principal} eigenvectors $\u_1,\ldots,\u_k$ 
i.e. those belonging to the $k$ largest eigenvalues of the data ``covariance matrix''
$\sum_t \x_t \x_t^\top$, and setting the projection matrix
to $\P = \sum_{i =1} ^k \u_i \u_i^\top$.
In this paper we choose the online version of PCA 
\citep{pca} as our paradigmatic matrix parameter problem
and we explore the core trade-off between generalization performance
and time efficiency per trial for this problem. 
In each trial $t=1,\ldots,T$, the online PCA algorithm chooses 
a projection matrix $\P_t$ of rank $k$
based on the previously observed points $\x_1,\ldots,\x_{t-1}$.
Then a next point $\x_t$ is revealed and the algorithm receives gain $\|\P_{t} \x_t\|^2$.
The goal here is to obtain an online algorithm whose cumulative gain over trials 
$t=1,\ldots,T$ is close to the cumulative gain
of the best rank $k$ projection matrix chosen in hindsight after seeing all $T$ instances. 
The maximum difference between the cumulative gain of the best off-line comparator and 
the cumulative gain of the algorithm 
and is called the (worst-case) \emph{regret}. 

If you use the principal eigenvectors $\u_i$ ($1\le i\le
k$) as the parameters, then the gain is formidably non-convex. However the key insight
of \citep{var,pca} is the observation 
that the seemingly quadratic gain $\|\P\x_t\|^2$ 
is a {\em linear function} of the projection matrix $\P$ 
when the data is expressed in terms of the 
\emph{matrix instance} $\x_t\x_t^\top$ rather than vector instance $\x_t$:
\[ \|\P \x_t \|^2 ~=~ \x_t \P^2 \x_t
~\stackrel{\P^2=\P}{=}~ \x_t \P \x_t ~=~ \tr(\P \,\x_t \x_t^\top).\]
Good algorithms hedge their bets by predicting with a random projection matrix. 
In that case $\EX\, [ \|\P \x_t \|^2]$ becomes $\tr(\EX[\P]\, \x_t\x_i^\top)$.
Thus it is natural to use mixtures $\EX[\P]$ of rank $k$ projection
matrices as the parameter matrix of the algorithm. 
Such mixture are positive definite matrices of trace $k$ with eigenvalues capped at $1$.
The gist is that the gain is now linear in this alternate parameter matrix 
and the non-convexity has been circumvented.
This observation is the starting point for lifting known online learning algorithms
for linear gain/loss on vector instances to the matrix domain,
which resulted in the Matrix Exponentiated Gradient (MEG) algorithm
\citep{meg,semidefinite,pca}, as well as the (matrix) Gradient Descent (GD) algorithm 
\citep{AroraNIPS,AroraAllerton,NieALT}.
Both algorithms are motivated by trading off a Bregman divergence against the gain,
followed by a Bregman projection onto the convex hull of
rank $k$ projection matrices which is our parameter space.
The worst-case regret of these algorithms for online PCA is
optimal (within constant factors). Furthermore MEG remains optimal for a generalization of
the PCA problem to the \emph{dense instance} case 
in which the ``sparse'' rank one outer products $\x_t\x_t^\top$ of
vanilla PCA are generalized to positive definite matrices $\X_t$ with bounded eigenvalues.
\citep{NieALT}.

Unfortunately, both algorithms 
require full eigendecomposition of the parameter matrix
at a cost of $O(n^3)$ per trial.%
\footnote{For the rank one instances $\x_t\x_t^\top$ of PCA the update of the eigenvalues takes 
$O(n^2)$ \cite{rankone} per trial. 
However the update of the eigensystem remains $O(n^3)$.}
It was posed as an open problem \citep{pcaopen}
whether there exists an algorithm with good regret guarantees
requiring time comparable with that of finding the top $k$
eigenvectors.
The latter operation operation can be done efficiently by means of
e.g. power iteration based methods \citep{arnoldi,lanczos}:
It essentially requires time $O(kn^2)$, which is much less
than the cost of a full eigendecomposition 
in the natural case when $k \ll n$.
This operation is also used by the simple \emph{Follow the Leader} 
algorithm which predicts with
the $k$ principal components of the current covariance matrix. 
This algorithm performs well when the data is i.i.d. but
can be forced to have large regret on worst-case data.

In this paper, we provide an algorithm based on the \emph{Follow the Perturbed Leader} (FPL) approach,
which perturbs the cumulative data matrix by adding a random symmetric noise matrix,
and then predicts with the $k$ principal components of the current perturbed covariance matrix. 
The key question is what perturbation to use and whether
that exists a perturbation for which FPL achieves close to
optimal regret.
In the vanilla vector parameter based FPL algorithm \citep{fpl}, 
exponentially distributed perturbation lead to optimal
algorithms when the perturbations is properly scaled.
We could apply the same perturbations to the eigenvalues
of the current parameter matrix and achieve optimal regret.
However this approach requires us to eigendecompose the
current parameter matrix and this defeats the purpose.
We need to find a perturbation that requires $O(n^2)$ time
to compute instead of $O(n^3)$.
We use a random symmetric Gaussian matrix
(a so called \emph{Gaussian orthogonal ensemble}), 
which consists of entries generated i.i.d. from a Gaussian distribution. 
Our approach is more similar to the recent algorithm based on 
\emph{Random Walk Perturbation} \cite{rwp}
and can be considered as a matrix generalization thereof, drawing connections to Random Matrix Theory \cite{randommatrix}.
Calculation of our random noise matrix requires $O(n^2)$ 
and hence the total computational time is dominated by
finding the $k$ principal components of the perturbed matrix. 
At the same time, our algorithm achieves $O(n^{1/4} \sqrt{kT})$
worst-case regret for online PCA (sparse instances) and $O(k\sqrt{nT})$ worst-case regret for dense instance case. Comparing to the minimax regrets
$\Theta(\sqrt{kT})$ and $\Theta(k \sqrt{T \log n})$ in the sparse and dense cases, respectively, we are only a factor of $O(n^{1/4})$ and
$O(n^{1/2})$ off from the optimum, respectively.

Our approach can be considered a generalization of the \emph{Random Walk Perturbation} (RWP) algorithm \cite{rwp}
to the matrix domain. In RWP, an independent Bernoulli coin flip is added to each component of the loss/gain vector, the process 
which can be closely approximated (through Central Limit Theorem) by Gaussian perturbations with variance growing linearly in $t$.
This is also the case of our algorithm, where we use a symmetric matrix with i.i.d. Gaussian-distributed entries with variance also growing
linearly in $t$.
Our analysis, however, resorts to properties of random matrices, e.g. expected maximum eigenvalue, which leads to worse regret bounds
than in the vector case.
Interestingly, comparing to \cite{rwp} we get rid of additional $O(\log T)$ factor in the regret.

\paragraph{Related work.}
Online PCA, in the framework considered here, was introduced in \cite{meg} and independently in \cite{semidefinite},
along with Matrix Exponentiated Gradient algorithm. The problem of finding efficient algorithms which avoid full eigendecomposition
was posed as an open problem by \cite{pcaopen}. An efficient algorithm for PCA based on Online Gradient Descent was proposed
in \cite{AroraNIPS,AroraAllerton}, but the main version of the algorithm (\emph{Matrix Stochastic Gradient, MSG}) still requires $O(n^3)$ in the worst case,
while a faster version (\emph{Capped MSG}) operates on low-rank deterministic parameter matrix, which can be shown to have
regret linear in $T$ in the adversarial setting. The most closely related to our work is \cite{GarberICML}, in which several
algorithms are proposed for learning the top eigenvector (i.e., the simplest case $k=1$ of online PCA):
based on online Franke-Wolfe method \cite{frankewolfe},
and based on FPL approach with entry-wise uniform perturbation, exponentially-distributed perturbation, and a perturbation
based on a sparse rank-one random matrix $\v \v^\top$ composed of a random Gaussian vector $\v$. Except for the last algorithm,
all the other approaches have regret guarantees which are inferior comparing to our method,
either in terms of dependence on $T$ or dependence on $n$, or both. The method based on rank-one perturbation achieves regret
bound $O(\sqrt{nT})$ which is the same as ours in the dense instance case. It is not clear whether this method would benefit anyhow
from sparsity of instance matrices (as in the standard online PCA), and whether it would easily generalize to $k > 1$ case.

There are other formulations of online PCA problem. For instance, \cite{GarberSODA} aims at finding low
dimensional data representation in a single pass, with the goal of good reconstruction guarantees using only a small number of dimensions.
On the other hand, \cite{incrementalPCA,stochasticPCA} consider online PCA in the stochastic optimization setting. However, the algorithm
considered therein are not directly applicable in the adversarial setting studied in this work.

\paragraph{New conjecture.}
Our algorithm is efficient, but its
suboptimal regret is due to
the fact that the noise matrix does not adapt to the
eigensystem of cumulative covariance matrix. 
What $O(n^2)$ perturbation can we use that does adapt
to the eigensystem of the current covariance matrix?
A clear candidate is to use \emph{Dropout Perturbation}. 
In the vector case this perturbation independently 
at random zeros out each component of the gain/loss vector in each trial
\citep{dropout} and achieves optimal regret without having
to tune the magnitude of the perturbations. 
In the matrix case it would be natural
to independently zero out components of the instance matrix
when expressed in the eigensystem of the current loss matrix.
However these approaches again require eigendecompositions.

A new variant is to skip at trial $t$ the entire
instance matrix with probability half, i.e. at trial $t$
predict with the $k$ principal components of the
following perturbed current covariance matrix 
$$\sum_{q=1}^{t-1} \alpha_t \X_t,$$ 
where the $\alpha_t$ are Bernoilli coin flips with probability half. 
We call this the {\em Follow the Skipping Leader} algorithm
because it skips entire instances $\X_t$ with probability
half. It is easy to maintain this perturbed covariance
matrix in $O(n^2)$ time per trial. 
\iffalse
Since a sum of Bernoulli random variables quickly
approaches a Gaussian distribution, this method is a natural
extension of our Gaussian perturbation algorithm in which the noise adapts to the data.
\fi
However, unfortunately already
in the vector parameter case this algorithm can be forced
to have a gravely suboptimal linear regret in $n$ \citep{ce}.
The counter example requires dense loss vectors.
When lifting this counter example to the matrix setting
then the regret can still be forced to be linear with
sparse instance $\x_t\x_t^\top$.
However we conjecture that the time efficient Follow the Skipping Leader algorithm achieves the
optimal regret for standard PCA with sparse instance.
This is because in PCA regret is naturally measured w.r.t.
the maximum gain of the best rank $k$ subspace and not the loss. 
Note that that this type of problem is decidedly not
symmetric w.r.t. gain and loss (See \cite{NieALT} for an
extended discussion).
\iffalse
Our conjecture is supported by the following two observations:
\begin{itemize}
\item
If all the data shares the same eigensystem, 
i.e. the instance vectors are always one of the $n$ standard basis vectors,
 then the matrix problem reduces to a known vector setting of prediction 
with expert advice with unit gain vectors. 
In this case the Follow the Skipping Leader algorithm
coincides with the Dropout Pertubation algorithm
which is known to be optimal for the vector setting.
\item
The regret bounds proven for the algorithms in the matrix case 
 are often the same as those provable for the original classical vector case \cite{meg,var,pca}. 
 In other words, the case in which all the data share the same eigensystem (diagonal case) seems to be essentially the
 hardest case in the matrix setting.
\end{itemize}
\fi

Finally we also conjecture the regret bounds achieved by the
algorithm of this paper (Gaussian perturbations)
is the best you can achieve with rotation invariant noise
and knowing this fact would be interesting in its own right.

%Finally, we also relate our algorithm to a recently introduced FPL algorithm based on \emph{dropout perturbation} \cite{dropout},
%in which each data point is perturbed by means of \emph{multiplicative}, rather than additive, noise. Generalizing dropout perturbation
%to the matrix domain would result in an algorithm with (approximately) Gaussian noise adapted to the eigensystem of the data.
%Despite having a counterexample of $\Omega(\sqrt{nT})$ regret for dense instance case, 
%We conjecture that the dropout algorithm would achieve the (asymptotically) optimal worst-case regret for online PCA.

\section{Problem setting}

In the online PCA, in each trial $t=1,\ldots,T$, the algorithm probabilistically chooses a projection matrix
$\P_t \in \RR^{n \times n}$ of rank $k$. Then a point $\x_t \in \RR^n$ is revealed and the algorithm receives
\emph{gain} $\|\P_t \x_t \|^2= \tr(\P_t \x_t \x_t^\top)$.
Note again that the
gain is linear in $\P_t$ and in \emph{instance matrix} $\x_t \x_t^\top$. 
This observation calls for generalization of the online PCA 
in which the instance matrix is any positive definite matrix $\X_t$ 
(with bounded eigenvalues) and the gain becomes $\tr(\P_t\X_t)$.
We call this the \emph{dense instance} case 
as opposed to standard online PCA, which we call \emph{sparse instance} case. 

In the above protocol, the algorithm is allowed to choose its $k$ 
dimensional subspace $\P_t$ probabilistically. 
Therefore we use expected gain $\EX [\tr(\P_t\X_t)]$
as the evaluation of the algorithm's performance, where the expectation is with respect to
the internal randomization of the algorithm. %\footnote{There is no randomness in the data considered.}.
The \emph{regret} of the algorithm is then the difference between 
the cumulative gain of the best off-line rank $k$ projector and the the cumulative gain of the algorithm (due to linearity of gain,
no randomization is necessary when considering the best off-line comparator):
\begin{equation}
\REG ~= \max_{P \in \PP} \left\{ \sum_{t=1}^T \tr(\P\X_t) 
\right\} ~-~
\sum_{t=1}^T \EX [\tr(\P_t\X_t)] 
~=~ \lambda_{1:k} \left(\X_{\leq T} \right) -
\sum_{t=1}^T \EX [\tr(\P_t\X_t)],
\label{eq:regret-def}
\end{equation}
where $\PP$ denotes the set of all rank $k$ projectors, $\X_{\leq t} = \sum_{q=1}^t \X_q$ is the cumulative data matrix, and
$\lambda_{1:k}(\X) = \sum_{i=1}^k \lambda_i(\X)$ denotes the sum of top $k$ eigenvalues of $\X$.
The goal of the algorithm is to have small regret for any sequence of instance matrices. %For the sake of simplicity, we only
%consider \emph{oblivious} adversary, i.e. the data sequence is arbitrary but does not depend on the internal random variables of the algorithm,
%but it is not hard to generalize the algorithm to the non-oblivious case.
Since the regret
naturally scales with the eigenvalues of $\X_t$, we assume for the sake of simplicity that all eigenvalues
of $\X_t$ are bounded by $1$ (i.e. for each $t=1,\ldots,T$, the spectral norm of $\X_t$, $\|\X_t\|_{\infty} \leq 1$).

We note that due to linearly of the gain, $\EX [\tr(\P_t\X_t)] = \tr(\EX [\P_t]\X_t)$, so that
the algorithm's gain is fully determined by $\EX[\P_t]$, a convex combinations of rank $k$ projection matrices.
Hence, the parameter set of the algorithm can be equivalently taken as $\WW = \mathrm{conv}(\PP)$,
a convex hull of $\PP$, which is
a set of positive definite matrices with trace $k$ and all eigenvalues not larger than $1$ \cite{pca}. This is the
key idea behind MEG and GD algorithms, which maintain the uncertainty about projection matrix by means
of a parameter $\W_t \in \WW$,  update their parameter 
by minimizing a trade-off between a divergence of the new and old parameter
and the gain/loss of the new parameter on the current instance,
while constraining the new parameter to lie in the parameter set $\WW$. While predicting,
the algorithm chooses its projection matrix $\P_t$ 
by sampling from this mixture $\W_t$ \cite{pca}.% Note that the convex
%combination with the maximum gain occurs at a ``pure'' projection matrix. 

\section{The algorithm}

Our algorithm belongs to a class of \emph{Follow the Perturbed Leader} (FPL) algorithms,
which are defined by the choice:
\[
 \P_t = \argmax_{\P \in \PP} \left\{ \tr(\P (\X_{< t} + \N_t)) \right\},
\]
where $\X_{<t} = \sum_{q < t} \X_q$ is the cumulative data matrix observed so far, while $\N_t$
is the symmetric noise matrix generated randomly by the algorithm\footnote{Note that if the algorithm plays against 
\emph{oblivious} adversary, it is allowed to generate the noise matrix once in the first trial and then reuse it
throughout the game.} and w.l.o.g. we assume $\EX[\N_t] = \mathbf{0}$. Perturbing the cumulative data matrix is necessary as one can easily show that
any deterministic strategy (including Follow the Leader obtained by taking $\N_t = \mathbf{0}$) can be forced
to have regret linear in $T$.

Define a ``fake'' prediction strategy:
\[
\widetilde{\P}_t =  \argmax_{\P \in \PP} \left\{ \tr(\P (\X_{\leq t} + \N_t)) \right\},
\]
which acts as FPL, but adds the current instance $\X_t$ to the cumulative data matrix, and hence does not corresponds
to any valid online algorithm. What follows is a standard lemma for bounding the FPL regret, adapted to the matrix case:
\begin{lemma}
\label{l:be-the-leader}
We have:
\[
\REG \leq \sum_t \EX\left[\tr((\widetilde{\P}_t - \P_t) \X_t) \right] + \sum_t \EX\left[\lambda_{1:k}(\N_t - \N_{t-1}) \right],
\]
\end{lemma}
A vector version of Lemma \ref{l:be-the-leader} can be find in standard textbooks on online learning (see, e.g., \cite{book}).
Since adaptation to the matrix case is rather straightforward, we defer the proof to the Appendix.

We now specify the noise matrix our algorithm employs. Let $\G$ be an $n \times n$ matrix such that each entry is generated i.i.d.
from a Gaussian distribution, i.e. $\G_{ij} \sim \NN(0,\sigma^2)$. We define the noise matrix of our algorithm as:
\[
 \N_t = \sqrt{t} \N, \qquad \text{where~~} \N = \frac{1}{2}\left(\G + \G^\top\right).
\]
Note that $\N$ is a symmetrized version of $\G$, and $\N_t$ multiplies $\N$ by $\sqrt{t}$ so that the variance of each entry in $\N_t$
grows linearly in $t$. Interestingly, distribution of $\N$ is known as \emph{Gaussian orthogonal ensemble} in the Random Matrix Theory
\cite{randommatrix}. The algorithm uses the variable noise rate and hence does not require any tuning for the time horizon $T$.
We still have a single parameter $\sigma^2$, but that parameter is only chosen based on the sparseness of the instance matrix.

Note that according to rules for summing Gaussian variables,
we can also express $\N_t$ as a sum of $t$ independent copies of $\N$,
$\N_t = \sum_{q=1}^t \N^{(q)}$. We thus get an equivalent picture of our algorithm in which in each trial, an independent noise
variable $\N^{(t)}$ generated from a fixed distribution is added to the current data instance $\X_t$, 
and then the action of the algorithm is based on the sum of perturbed data instances.
This pictures makes our approach similar to RWP algorithm and let us relate our algorithm to dropout perturbation in the next section.

We now show the main result of this paper, the regret bound of the algorithm based on Gaussian perturbation.
\begin{theorem}
Given the choice of the noise matrix $\N_t$ described above,
\begin{itemize}
 \item For dense instance, setting $\sigma^2 = 1$ gives: 
 \[
 \REG ~\leq~ 2 k \sqrt{n T}.  
 \]
 \item For sparse instance, setting $\sigma^2 = \frac{1}{k \sqrt{n}}$ gives:
 \[
 \REG ~\leq~ 2 n^{1/4} \sqrt{k T}.
 \]
\end{itemize}
\end{theorem}
\begin{proof}
We apply Lemma \ref{l:be-the-leader} and bound both sums on the right hand side separately. We start with the second sum.
We have:
\[
\sum_t \EE{\lambda_{1:k}(\N_t - \N_{t-1})} = \sum_t (\sqrt{t} - \sqrt{t-1}) \EE{\lambda_{1:k}(\N)} = \sqrt{T}\EE{\lambda_{1:k}(\N)} 
\]
It follows from Random Matrix Theory (see, e.g., \cite{DavidsonSzrek2001})  that the largest eigenvalue of a matrix generated from
a Gaussian orthogonal ensemble is of order $O(\sqrt{n})$, specifically:
\[
\EE{\lambda_{\max}(\N)} \leq \sqrt{n\sigma^2}.
\]
Therefore,
\[
\EE{\lambda_{1:k}(\N)} \leq k \EE{\lambda_{\max}(\N)} \leq k \sqrt{n \sigma^2},
\]
so that the second sum is bounded by $k \sqrt{n T \sigma^2}$.

Let us now bound the first sum.
First, note that $\N_{ij} \sim \NN(0,\sigma^2/2)$ for $i \neq j$ and $\N_{ii} \sim \NN(0,\sigma^2)$. 
This means that the joint density $p(\N) = p(\N_{11},\ldots,\N_{nn})$ is proportional to:
\[
p(\N) ~\propto~ \exp\bigg\{-\frac{1}{2\sigma^2} \bigg( \sum_{i > j} 2\N_{ij}^2 + \sum_i \N_{ii}^2 \bigg)\bigg\}
~=~ \exp \left\{ -\frac{1}{2\sigma^2} \tr(\N^2)\right \}.
\]
Similarly, the joint density of $\N_t$ is proportional to:
\[
p_t(\N_t) \propto \exp \left\{ \frac{1}{2 t\sigma^2} \tr(\N_t^2)\right \}.
\]
For any symmetric matrix $\A$, define:
\[
 \P(\A) = \argmax_{\P \in \PP} \left\{ \tr(\P \A) \right\},
\]
so that $\P_t = \P(\X_{<t} + \N_t)$ and $\widetilde{\P}_t = \P(\X_{<t} + \N_t + \X_t)$.
Furthermore, define a function:
\[
 f_t(s) = \EE{\tr(\P(\X_{<t} + \N_t + s \X_t)\, \X_t)}.
\]
Note that $f_t(0) = \EE{\tr(\P_t \X_t)}$ and $f_t(1) = \EE{\tr(\widetilde{\P}_t \X_t)}$.
In this notation,
\[
\sum_t \EE{ \tr((\widetilde{\P}_t - \P_t) \X_t) } = \sum_t (f_t(1) - f_t(0)),
\]
so it remains to bound $f_t(1) - f_t(0)$ for all $t$.
We have:
\begin{align*}
f_t(s) &= \int \tr(\P(\X_{<t} + \N_t + s\X_t) \, \X_t) p_t(\N_t) \dif \N_t\\
&= \int \tr(\P(\X_{<t} + \N_t) \, \X_t) p_t(\N_t - s \X_t) \dif \N_t,
\end{align*}
which follows from changing the integration variable from $\N_t$ to $\N_t - s \X_t$.
Since by H{\"o}lder's inequality:
\begin{equation}
\tr(\P(\X_{<t} + \N_t) \, \X_t) ~\leq~ \tr(\P(\X_{< t} + \N_t)) \cdot \|\X_t\|_{\infty}
~=~ k \|\X_t\|_{\infty} ~\leq~ k,
\label{eq:holders}
\end{equation}
and since $p_t$ is the density of Gaussian distribution, 
it can easily be shown by using standard argument based on Dominated Convergence Theorem
that one can replace the order of differentiation w.r.t. $s$ and integration w.r.t. $\N_t$.
This means that $f_t(s)$ is differentiable and:
\begin{align*}
 f'_t(s) &~=~ \int \tr(\P(\X_{<t} + \N_t) \, \X_t) \frac{\dif p_t(\N_t - s \X_t)}{\dif s} \dif \N_t \\
 &~=~ \frac{1}{t\sigma^2}\int \tr(\P(\X_{<t} + \N_t) \, \X_t) \tr((\N_t - s\X_t) \X_t) p_t(\N_t - s \X_t) \dif \N_t \\
 &~=~ \frac{1}{t\sigma^2} \int \tr(\P(\X_{<t} + \N_t + s \X_t) \, \X_t) \tr(\N_t \X_t) p_t(\N_t)  \dif \N_t \\
 &~\leq~ \frac{r}{t\sigma^2} \int \left(\tr(\N_t \X_t) \right)_+ p_t(\N_t) \dif \N_t, 
\end{align*}
where $(c)_+ = \max\{c,0\}$, and $r = k$ in the dense instance case, while $r = 1$ in the sparse instance case.
The last inequality follows from the same argument as in (\ref{eq:holders}) when the instances are dense, and from the opposite application
of H{\"o}lder's inequality when the instances are sparse, i.e. for sparse instance $\X_t = \x_t \x_t^\top$ with $\|\x_t\| = 1$, and any $\P$:
\[
\tr (\P \, \x_t \x_t^\top) ~\leq~ \tr(\x_t \x_t^\top) \cdot \|\P\|_{\infty} ~=~ \|\x_t\|^2 \cdot 1 ~=~ 1.
\]
Denote:
\[
z ~=~ \tr(\N_t \X_t) = 2 \sum_{i > j} (\N_t)_{ij} (\X_t)_{ij}
+ \sum_i (\N_t)_{ii} (\X_t)_{ii}.
\]
Using summation rules for Gaussian variables:
\[
z ~\sim~ \NN\left(0,\, 2t \sigma^2 \sum_{i > j} (\X_t)^2_{ij} + 
t \sigma^2 \sum_i (\X_t)^2_{ii} \right)
~=~ \NN \left(0, t \sigma^2 \tr(\X^2) \right),
\]
so that:
\begin{align*}
f'_t(s)
 &~\leq~ \frac{r}{t\sigma^2} \int (\tr(\N_t \X_t))_+ p_t(\N_t) \dif \N_t \\
 &~=~ \frac{r}{t\sigma^2} \mathbb{E}_{z \sim \NN(0,t \sigma^2 \tr(\X^2))}[(z)_+] \\
 &~=~ \frac{r}{\sqrt{t\sigma^2}} \sqrt{\tr(\X^2)} \; \mathbb{E}_{z \sim \NN(0,1)} [(z)_+] \\
 &~=~ \frac{r}{\sqrt{2 \pi t\sigma^2}} \sqrt{\tr(\X^2)}.
\end{align*}

By the mean value theorem,
\[
 f_t(1) - f_t(0) = f_t'(s), \qquad \text{for some~} s \in [0,1],
\]
which implies:
\[
\EE{ \tr((\widetilde{\P}_t - \P_t) \X_t) } \leq  \frac{r}{\sqrt{2 \pi t\sigma^2}} \sqrt{\tr(\X_t^2)}.
\]
Summing over trials and using $\sum_{t=1}^T 1/\sqrt{t} \leq 2\sqrt{T}$, we bound the first sum in Lemma \ref{l:be-the-leader}
by $r\sqrt{\frac{2 T \max_t \tr(\X_t^2) }{\pi \sigma^2}}$. Using the bound on the second sum, we get that:
\[
\REG
~\leq~ r\sqrt{\frac{2 T \max_t \tr(\X_t^2) }{\pi \sigma^2}}
+ k\sqrt{nT \sigma^2}.
\]
The proof is finished by noticing that $\tr(\X^2) \leq n$ and $r=k$ for dense instances, 
while $\tr(\X^2) = 1$ and $k=1$ for sparse instances.
\end{proof}

Comparing the regret with values of the minimax regret
$\Theta(\sqrt{kT})$ in the standard online PCA setting (sparse instance case), and
$\Theta(k \sqrt{T \log n})$ in the dense instance case \cite{NieALT},
we see that the algorithm presented here is suboptimal by a factor of $O(n^{1/4})$ in the online PCA setting,
and by a factor of $O(\sqrt{n} / \log n)$ in the dense instances setting.

\section{Conclusions}

In this paper, we studied the online PCA problem and its generalization to the case of dense instance matrices.
While there are algorithms which essentially achieve the minimax regret, such as Matrix Exponentiated Gradient
or (Matrix) Gradient Descent, all these methods take $O(n^3)$ per trial, because
they require full eigendecomposition of the data matrix.
We proposed an algorithm based on Follow the Perturbed Leader approach, which uses as a perturbation
a random symmetric Gaussian matrix. The algorithm
avoids full eigendecomposition and only requires calculating the top $k$ eigenvectors.
Hence, prediction takes $O(k n^2)$, while the algorithm achieves the worst-case regret 
which is only $O(n^{1/4})$ close to the minimax regret for standard online PCA setting,
and $O(\sqrt{n})$ close to minimax regret for generalization of online PCA to
a dense instance matrices. 
Finally, we raised an open question, whether a more adaptive version of our algorithm, based on dropout perturbation,
would achieve the minimax regret.

\subsubsection*{Acknowledgments}
Wojciech Kot{\l}owski was supported by the Polish National Science Cente grant\\
2013/11/D/ST6/03050.

\appendix

\section{Proof of Lemma~\ref{l:be-the-leader}}
We have:
\begin{align*}
 \lambda_{1:k}(\X_{\leq t} + \N_t) &~=~ \max_{\P \in \PP} \left\{ \tr(\P (\X_{\leq t} + \N_t)) \right\} \\
 &~=~ \tr(\widetilde{\P}_t(\X_{\leq t} + \N_t)) \\
 &~=~  \tr(\widetilde{\P}_t(\X_t + \N_t - \N_{t-1})) + \tr(\widetilde{\P}_t(\X_{< t} + \N_{t-1})) \\
 &~\leq~ \tr(\widetilde{\P}_t(\X_t + \N_t - \N_{t-1})) + \lambda_{1:k}(\X_{< t} + \N_{t-1}),
\end{align*}
so that:
\[
\lambda_{1:k}(\X_{\leq t} + \N_t) - \lambda_{1:k}(\X_{< t} + \N_{t-1}) \leq \tr(\widetilde{\P}_t(\X_t + \N_t - \N_{t-1})). 
\]
Summing over trials $t=1,\ldots,T$, the terms on the left-hand side telescope and, defining $\N_0 = \mathbf{0}$, we get:
\[
\lambda_{1:k}(\X_{\leq T} + \N_T) ~\leq~ \sum_{t=1}^T \tr(\widetilde{\P}_t(\X_t + \N_t - \N_{t-1})).
\]
Since $\lambda_{1:k}(\cdot)$ is convex as a maximum over linear functions, Jensen's inequality implies
$\lambda_{1:k}(\X_{\leq T}) \leq \EE{\lambda_{1:k}(\X_{\leq T} + \N_T)}$ and hence:
\begin{align*}
\lambda_{1:k}(\X_{\leq T}) 
&~\leq~ \sum_{t=1}^T \EE{\tr(\widetilde{\P}_t(\X_t + \N_t - \N_{t-1}))} \\
&~\leq~ \sum_{t=1}^T \EE{\tr(\widetilde{\P}_t \X_t)} + \sum_{t=1}^T \EE{\max_{\P \in \PP} \left\{\N_t - \N_{t-1} \right\}} \\
&~=~ \sum_{t=1}^T \EE{\tr(\widetilde{\P}_t \X_t)} + \sum_{t=1}^T \EE{\lambda_{1:k}(\N_t - \N_{t-1})}.
\end{align*}
The lemma follows by plugging the inequality above into the definition of the regret (\ref{eq:regret-def}).

%\newpage
\bibliography{pap}

\end{document}

%% file: def.tex
\usepackage{amssymb,amsmath}
\usepackage{etoolbox}
\usepackage{color}
\usepackage{multirow}
\usepackage{framed}
\usepackage{bm}
\usepackage[normalem]{ulem}

\newcommand{\EX}{\mathbb{E}}

\newcommand{\REG}{\mathcal{R}}

\newcommand{\CF}[1]{\boldsymbol{#1}}

\definecolor{blue}{rgb}{0,0,0.8}
\definecolor{magenta}{cmyk}{0,1,0,0}
\definecolor{lightgrey}{rgb}{0.5,0.5,0.5}
\definecolor{DarkGreen}{rgb}{0.0, 0.5, 0.0}

\newcommand{\argmax}{\operatornamewithlimits{argmax}}

\renewcommand{\u}{{\boldsymbol{u}}}

\renewcommand{\v}{{\boldsymbol{v}}}

\newcommand{\WW}{\boldsymbol{\mathcal{W}}}
\newcommand{\PP}{\boldsymbol{\mathcal{P}}}

\newcommand{\RR}{\mathbb{R}}

\DeclareMathOperator{\tr}{tr}

\ifundef{\half}{
  \newcommand{\half}{{\tfrac{1}{2}}}
}{}

\newcommand{\A}{{\CF{A}}}

\newcommand{\N}{{\CF{N}}}
\newcommand{\W}{{\CF{W}}}

\newcommand{\X}{{\CF{X}}}

\newcommand{\G}{{\CF{G}}}
\renewcommand{\P}{{\CF{P}}}

\newcommand{\R}{\CF{R}}

\newcommand{\x}{{\boldsymbol{x}}}
\newcommand{\e}{{\boldsymbol{e}}}